\documentclass{article}
\usepackage{moi}

\author{
  Sajjad  {Hashemian}\\
  \texttt{shashemianm@math.uk.ac.ir}
  \and
  Mohammad Saeed Arvenaghi\\
  \texttt{m\_arvenaghi@mathdep.iust.ac.ir}
  \and
  Ebrahim Ardeshir-Larijani\\
  \texttt{larijani@iust.ac.ir}
}
\title{Optimal Bound for PCA with Outliers using Higher-Degree Voronoi Diagrams}

\begin{document}
\maketitle

\begin{abstract}
In this paper, we introduce new algorithms for Principal Component Analysis (PCA) with outliers. Utilizing techniques from computational geometry, specifically higher-degree Voronoi diagrams, we navigate to the optimal subspace for PCA even in the presence of outliers. This approach achieves an optimal solution with time complexity of \(n^{d+\mathcal{O}(1)}\text{poly}(n,d)\).  
Additionally, we present a randomized algorithm with complexity \(O(n^r \log(1/\delta) / C(d,r,\alpha))\text{poly}(n,d)\). Our approach leverages properties of high-dimensional spaces and the separation condition of outliers to efficiently recover the optimal subspace. Our results demonstrate that higher-degree Voronoi diagrams, combined with probabilistic subspace selection techniques, provide an effective and scalable solution for PCA with outliers.
\end{abstract}

\section{Introduction}

	Principal Component Analysis (PCA) is one of the most widely used techniques of dimensionality reduction algorithms in machine learning. In PCA, the goal is to find the best low-rank approximation of a given matrix. This can be conveniently achieved using singular value decomposition (SVD) of the given matrix \cite{pca}. SVD decomposes the matrix into three matrices: a left singular matrix, a diagonal matrix of singular values, and a right singular matrix. The low-rank approximation can be obtained by keeping only the top $r$ singular values and their corresponding singular vectors, where $r$ is the desired rank of the low-rank approximation. This is shown in ~(\ref{eq:1}).
\begin{equation}\label{eq:1}
\begin{split}
& \texttt{Input:}\:\text{X}\in \mathcal{M}_{n\times d}(\mathbb{R}), r\in\{1,2,\cdots,d\}\\
& \arg\min_{L}  ||\text{X}-L||_F^2\ \\
& \texttt{Subject}\:\texttt{to:}\: \text{rank} (L)\leq r   
\end{split}
\end{equation}	
Here $||\text{A}||_F^2 = \sum_{i,j}{a_{i,j}^2} $ is the square of the Frobenius norm of matrix A.

PCA has many applications in machine learning, such as image processing \cite{gonzalez2002digital, barrett2013foundations}, signal processing \cite{castells2007principal,MENG2021108503}, neuroscience \cite{chapin1999nicolelis,jirsa1994theoretical}, and quantitative finance \cite{pasini2017principal, yang2015application}. 
It is often used to reduce the dimensionality of high-dimensional data while preserving the most important information. This leads to improvement in the efficiency and accuracy of machine learning models that work with large high-dimensional data.

In the presence of outliers, the classical PCA algorithm may not work very well, as they can significantly affect the results.
The \textsc{PCA with Outliers} problem attempts to improve the results of PCA by \textit{removing some data points} that are considered outliers. The goal is to recover a low-rank matrix with sparse errors. Here, the assumption is that the data points lie near an $r$-dimensional subspace, and the remaining points are distributed randomly, possibly as a result of distorted observations.
This problem can be formulated as follows: given a matrix $\text{X}\in \mathcal{M}_{n\times d}(\mathbb{R})$, where $n$ is the number of data points and $d$ is the dimensionality of the data, and a parameter $r$ representing the rank of the low-rank matrix, the goal is to find a low-rank matrix $L$ and a sparse error matrix $S$ such that $\text{X}=S+L$, where $\text{rank} (L)\leq r$ and $||S||_\#\leq k$, where $k$ is a parameter representing the number of outliers to be removed,
and  $||\cdot||_{\#} $ denotes the number of non-zero rows in a matrix, as shown in~(\ref{eq:2}).

\begin{equation}\label{eq:2}
\begin{split}
    & \texttt{Input:}\quad \text{X}\in \mathcal{M}_{n\times d}(\mathbb{R}), r\in\{1,2,\cdots,d\},k\in\{1,2,\cdots,n\}\\
    & \texttt{Output:}\quad \operatorname{argmin}_{S, L}  ||\text{X}-S-L||_F^2 \\
    & \texttt{Subject to:} \quad \text{rank} (L)\leq r, ||S||_{\#}\leq k
\end{split}   
\end{equation}

The \textsc{PCA with Outliers} problem, or robust PCA, has many applications in various fields of research. It can be used to remove noise or artifacts from images in medical imaging \cite{solomon2019deep,ahmadi2023detection}, detect anomalies or outliers in financial data in portfolio management \cite{crepey2022anomaly}, identify suspicious activities in video data \cite{bouwmans2018applications}, and remove noise or distortions from signals in speech recognition \cite{hung2018employing}. These applications demonstrate the versatility and usefulness of the \textsc{PCA with Outliers} problem in various fields of research.

\textbf{Our Results}
The contribution of this paper is as follows.

First, we propose a new algorithm for \textsc{PCA with Outliers} with time complexity \( n^{\mathcal{O}(d)} \text{poly}(n, d) \), which fills the gap identified in~\cite{refined} as this is the optimal bound. To achieve this, we leverage higher-degree Voronoi diagrams, which are defined in the next section.

Second, we develop a randomized algorithm with time complexity 
\(
O\Bigl( \frac{n^r \log(1/\delta)}{C(d,r,\alpha)}\text{poly}(n,d) \Bigr),
\)
where \(C(d,r,\alpha)\) is an explicit constant depending on \( d, r, \) and the separation parameter \( \alpha \). 
Our randomized algorithm leverages the \(\alpha\)-gap condition to ensure that inliers and outliers are well-separated, making it possible to identify the optimal subspace with high probability. We say the instance satisfies the \emph{$\alpha$-gap condition} when every outlier is at least $(1+\alpha)$ times farther from the optimal subspace $L^\star$ than every inlier (see Lemma \ref{alphagap}).
 Specifically, we prove that the probability of success is at least \(1 - \delta\), where \(\delta\) controls the failure rate over multiple independent trials. The algorithm works by iteratively sampling random subspaces and selecting the one that best fits the inliers.
This approach offers a scalable and conceptually intuitive solution to \textsc{PCA with Outliers}, particularly in high-dimensional settings where traditional deterministic methods are computationally infeasible. Our results indicate that higher-degree Voronoi diagrams, combined with randomized subspace selection, yield an optimal and practical solution for outlier-robust PCA.


\subsection{Related works}
To solve this problem exactly, one can trivially check all the possible $n\choose k$ ways to choose $k$ vectors as outliers and then perform PCA on the rest $n-k$ vectors and return the best result among all this as the solution of \textsc{PCA with Outliers}. However, not all these $n\choose k$ configurations are possible when you consider the geometric properties of the given vectors and possible subspaces.

In \cite{wright}, the authors formulated the \textsc{PCA with Outliers} as the optimization problem shown in (\ref{eq:3}).
\begin{equation}\label{eq:3}
\begin{split}
 & \min_{L,S} \Vert L \Vert_* + \lambda \Vert S \Vert_1 \\
 & \texttt{subject to:} \quad X = L + S
\end{split}    
\end{equation}
Where $\Vert \cdot \Vert_*$ denotes the nuclear norm of a matrix, which is the sum of its singular values, and $\Vert \cdot \Vert_1$ denotes the $\ell_1$ norm of a matrix, which is the sum of the absolute values of its entries. The parameter $\lambda$ controls the trade-off between the low-rank and sparse components of the data matrix. 

Also the \textsc{PCA with Outliers} problem can be solved using a variety of convex optimization techniques, such as the Alternating Direction Method of Multipliers or the Augmented Lagrangian Method.
The paper \cite{xu} proposes the Outlier Pursuit method to solve this convex problem. This algorithm is based on the observation that the sparse matrix $S$ can be decomposed into a sum of sparse vectors, each of which corresponds to an outlier. It iteratively identifies and removes the outliers from the data matrix, and then computes the low-rank component of the remaining data using PCA. The outliers are then added back to the low-rank component to obtain the final estimate of the data matrix. The mentioned algorithm requires the sparsity level to be known a priori, which can be a limitation in practice. Additionally, the algorithm assumes that the noise in the data matrix is Gaussian, which may not be true in practice.

To overcome the issues of the previous methods, \cite{chen} proposed the Manipulator Pursuit algorithm, which is more robust and flexible. This method can handle wider range of outlier models and noise distributions. It does not require the sparsity level to be known a priori as it uses data-driven approach to estimate the sparsity level and a nuclear norm regularization term to encourage low-rank solutions.

In \cite{refined}, the authors analyzed the computational complexity of the \textsc{PCA with Outliers} problem. They refined the lower bound on the time complexity of the problem by showing that it is $\Omega(n^d)$. Finally it is concluded that the lower bound is tight up to a polynomial factor, and also an algorithm for solving this problem in $n^{\mathcal{O}(d^2)}{\text{poly}(n,d)}$ is proposed using tools from real algebraic geometry. To be explicit, their result is summarized in the following proposition.

\begin{proposition}\label{pr:1} 
\textnormal{(\cite{basu_algorithms_2003}, Theorem 13.22)}
Consider $V$ as an algebraic set in $\mathbb{R}^d$ of real dimension $d'$, defined by $Q(X_1, \ldots, X_d) = 0$, where $Q$ is a polynomial in $\mathbb{R}[X_1, \ldots, X_d]$ of degree at most $b$. Let $\mathcal{P} \subset \mathbb{R}[X_1, \ldots, X_d]$ be a finite set of $s$ polynomials, each $P \in \mathcal{P}$ also having a degree at most $b$. Define $D$ as the ring generated by the coefficients of $Q$ and the polynomials in $\mathcal{P}$. Given $Q$, $d'$, and $\mathcal{P}$, there exists an algorithm that computes a set of points such that each non-empty cell of 
$V$ over $\mathcal{P}$ contains at least one point from the set.
This algorithm requires at most $s^{d'} b^{O(d)}$ arithmetic operations in $D$.

\end{proposition}  

\begin{proposition}

\label{lowerbound}
\textnormal{(\cite{refined}, Theorem 1)}
Solving \textsc{PCA with Outliers} is reducible to solving $n^{\mathcal{O}(d^2)}$ instances of PCA.

\end{proposition}

\begin{proof}
Consider an algebraic set $W = \mathbb{R}^{(d-r) \times d}$. For every $V$ in $W$, we define $S_{V}$ consists of $k$ non-zero rows, selected from $X$ based on those with the highest value of $\|x_i -\text{proj}_{V} x_i \|^2_F$, equivalent to $\|V x^T_i \|^2_F$. Meanwhile, $L_{V}$ represents the orthogonal projection of the rows of $(X - S_{V})$ onto $V$.

Now, let's contemplate the set of polynomials $ \mathcal{P} = \{P_{i,j}\}_{1\leq i<j\leq n} $ defined over $W$, where $$P_{i,j}(V) = \|V x^T_{i}\|_F^2 - \|V x^T_{j}\|_F^2.$$

The $k$ farthest points and the matrix $S_{V}$ remain consistent across all cells in the partition of $W$ over $\mathcal{P}$. Therefore, it suffices to select a point $V$ from each cell, calculate the outlier matrix $S_{V}$, and solve PCA for $(X - S_{V}, r)$.

Combining all steps, their algorithm functions as outlined below:

\begin{enumerate}
\item Utilize the procedure outlined in Proposition \ref{pr:1} to acquire a point $V_C$ from each cell $C$ of $W$ over $\mathcal{P}$.
\item For each $V_C$, compute the optimal $S_{V_C}$ by selecting the $k$ rows of $X$ with the highest value of $||V_C x^T_{i}||_F^2$. Construct the instance $(X - S_{V_C}, r)$ for PCA.
\item The solution to the original instance of \textsc{PCA with Outliers} is the optimal solution among all the solutions obtained from the constructed PCA instances.
\end{enumerate}

By running this algorithm, we can reduce the problem of \textsc{PCA with Outliers} to solving $\binom{n}{2}^{(d - r)d}2^{O(d)}$ instances of PCA.
\end{proof}

\begin{proposition}
\label{lowerbound}
\textnormal{\cite{refined}}
There is no $\omega$-approximation algorithm with running time $f(d)\cdot N^{{o}(d)}$ for any $\omega \geq 1$ and computable function $f$ unless Exponential Time Hypothesis (ETH) fails, where $N$ is the bit size of the input matrix $X$.
\end{proposition}

In \cite{fixedpca} new algorithms for the \textsc{PCA with Outliers} problem are proposed. The authors first present a ($1+\epsilon$)-approximation algorithm for the problem, which has a time complexity of $n^{\mathcal{O}(\epsilon^{-2}\cdot r\log r)}$. Then they presented a randomized algorithm for the problem, which outputs the correct answer with probability $1-\delta$ in time $2^{\mathcal{O}(r(\log k+\log \log n)(\log k+\log \log n+\log 1/\delta))}\text{poly}(n,d,1/\delta)$ under the $\alpha$-gap assumption. This assumption generally guarantees that the outliers must have a certain distance to the expected lower-dimensional subspace.


\section{Notation}
For a positive integer $n$, we use $[n]$ to denote the set ${1, 2, \cdots, n}$. The collection of all linear $r$-dimensional subspaces of $\mathbb{R}^d$ is represented as $\mathcal{L}_r(\mathbb{R}^d)$ and the distance of vector $x$ from subspace $V$ denoted by $\mathtt{dist}(x,V)$, is defined as $\inf_{v\in V}||x-v||$. Similarly, the distance between two subspaces $W$ and $W'$ is then $\Vert P_W-P_{W'}\Vert$ where $P_X$ denotes projection onto $X$ and $\Vert\cdot\Vert$ is the operator norm.
Also, the $\Vert \cdot \Vert_1$ denotes the $\ell_1$ norm of a matrix, which is the sum of the absolute values of its entries. Unless otherwise stated, $\mathcal{P}(A)$ denotes the power set of the set $A$.
\\
In this paper, we consider the \textsc{PCA with Outliers} problem, which takes an instance $(X,r,k)$ as input, where $X$ is a subset of $\mathbb{R}^d$ of size $n$, and $r,k\in \mathbb{N}$. The goal is to find a $r$-dimensional subspace $L$ that captures the most important information contained in the data points of $X$ while ignoring up to $k$ outliers. These outliers, which are rows of $X$, will be denoted by a (sparse) matrix $S$ of the same size, where only the rows corresponding to the \textit{inliers} are replaced with zeros. Throughout the paper, we will use matrix form and linear subspace spanned by rows of that matrix, interchangeably. Also, 
for $x\in\mathbb{R}^d$ and $\epsilon>0$, we denote by 
$B_\epsilon^d(x) \;=\;\{\,y\in\mathbb{R}^d:\|y - x\|_2 \le \epsilon\}$
the closed Euclidean ball of radius \(\epsilon\) centered at \(x\) in \(d\)-dimensional space.

\section{Voronoi diagrams} 

Consider a space $X$ (e.g. $\mathbb{R}^d$ or a subset of it), a set of objects $S$ (e.g. a subset of $X$), and a distance function $d: X\times S\to\mathbb{R}$ on this space. A Voronoi diagram partitions $X$ into regions, each corresponding to the set of points closest to an element of $S$. Each element of the set $S$ such as $s_i$ has a corresponding region, called a Voronoi cell, consisting of all the elements of the space $X$ closer to that element $s_i$ than to any other element $s_{j\neq i}$ in the set. Voronoi diagrams have many applications in various fields \cite{aurenhammer1991voronoi,okabe1992spatial}.
An example of a Voronoi diagram with a data set of size 15 is depicted in Figure~\ref{fig:vor}. 

\begin{figure}
    \centering
    \includegraphics[height = 5.5 cm, width = 5.5 cm]{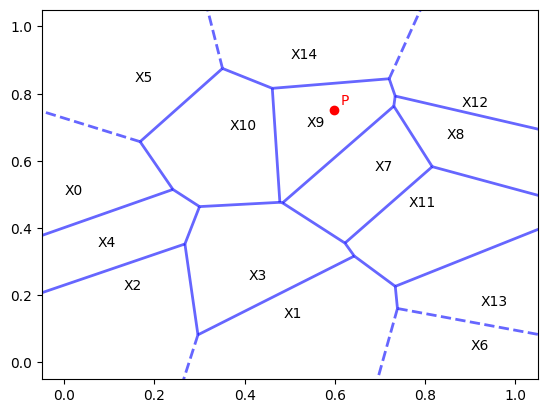}
    \caption{A Voronoi diagram with 15 data points that partition the 2D plane. The query point labeled with $P$ is closer to the point $X_9$ than any other points in the dataset, a property that all points surrounding the region of $X_9$ share as well.}
    \label{fig:vor}
\end{figure}

In the context of the \textsc{PCA with Outliers} problem, we use Voronoi diagrams to assign each data point to a subspace based on its distance to a set of candidate subspaces. This leads to solving the problem more efficiently by reducing the search space and avoiding the need to compute the distance to all possible subspaces.
In the following, we present the formal definition of Voronoi diagrams.

\begin{definition}
\label{VOD}
\textnormal{\textbf{(Voronoi Diagrams)}}
Let $\mathcal{S}\subset \mathbb{R}^d$ be a set of $n$ vectors in a $d$-dimensional space.
\textbf{Voronoi diagram} is defined as the partitioning $V\subset\mathcal{P}(\mathbb{R}^d)$ of $\mathbb{R}^d$ such that for 
a vector $s\in \mathcal{S}$:
$$
\forall x,y\in\mathbb{R}^d \quad x\sim y\iff \arg\min_{s\in S}\|x-s\|=\arg\min_{s\in S}\|y-s\|
$$
where $\mathcal{P}(A)$ is the power set of $A$.
\end{definition}

\begin{definition}
\label{VODOnHyperspaces}
Let $\mathcal{S}$ be a set of $n$ vectors in $\mathbb{R}^d$. We can define the partitioning $V\subset\mathcal{P}(\mathcal{L}_r(\mathbb{R}^d))$ of $\mathcal{L}_r(\mathbb{R}^d)$ such that for a vector $s\in \mathcal{S}$:
$$
\forall U,V\in \mathcal{L}_r(\mathbb{R}^d) \quad U\sim V\iff \arg\min_{s\in \mathcal{S}}\mathtt{dist}(x,U)=\arg\min_{s\in \mathcal{S}}\mathtt{dist}(x,V)
$$
\end{definition}

Although a normal Voronoi cell is defined as the set of elements closest to a single object in $S$, a $\mathtt{k}$-degree Voronoi cell is defined as the set of elements having a particular tuple of $k$ objects in $S$ as their nearest neighbors. They also subdivide space.
\\
Moreover, higher-degree Voronoi diagrams can be generated recursively. To generate the $\mathtt{k}$-degree Voronoi diagram from set $S$, start with the $\mathtt{(k - 1)}$th-degree diagram and replace each cell generated by $T = \{s_1, s_2,\cdots, s_{k-1}\}$ with a Voronoi diagram generated on the set $S \setminus T$.

\begin{definition}
\label{HigherDegreeVOD}
\textnormal{\textbf{(Higher Degree Voronoi Diagrams)}}
Let $\mathcal{S}\subset \mathbb{R}^d$ be a set of $n$ vectors in a $d$-dimensional space. We define the 
\textbf{$k$-degree Voronoi diagram} as a partitioning induced by a function $V:\mathcal{P}^k(\mathcal{S})\to\mathcal{P}(\mathbb{R}^d)$ such that:
$$
V(U)=\bigg\{ \: \:x \:\: \bigg| \:\:\max_{u\in U} \mathtt{dist}(x,u)\leq \min_{t\in \mathcal{S}\setminus U}\Vert x-u\Vert
\:\: \land \:\: \Vert x-u_1\Vert\leq \Vert x-u_2\Vert\leq\cdots\leq \Vert x-u_k\Vert \bigg\}  
$$
where $\mathcal{P}^k(A)$ is the set of all $k$-tuples with elements from $A$ and \(U=(u_1,u_2,\ldots,u_k)\) is an ordered \(k\)-tuple of points from \(\mathcal S\) (i.e.\ \(U\in\mathcal P^k(\mathcal S)\)).
\end{definition}

\begin{definition}
\label{VODOfHyperspacesOnPoints}
Let $\mathcal{S}\subset \mathbb{R}^d$ be a set of $n$ vectors in a $d$-dimensional space. We define the {$k$-degree Voronoi diagram of $r$-dimensional subspaces with respect to vectors $\mathcal{S}$} as a partitioning induced by a function $V:\mathcal{P}^k(\mathcal{S})\to\mathcal{P}(\mathcal{L}(\mathbb{R}^d))$ such that:

\[
\begin{split}
V(U)=\bigg\{ \: \:W \:\: \bigg|
&\:\:\max_{u\in U} \mathtt{dist}(u, W)\leq \min_{t\in \mathcal{S}\setminus U}\mathtt{dist}(t, W)
\\&\:\: \land \:\: \mathtt{dist}(u_1,W)\leq \mathtt{dist}(u_2,W)\leq\cdots\leq \mathtt{dist}(u_k,W) \bigg\}  
\end{split}
\]

where $\mathcal{P}^k(A)$ is the set of all $k$-tuples with elements from $A$ and $\mathtt{dist}(x, y)$ is the Euclidean distance function and $\mathcal{P}(\mathcal{L}(\mathbb{R}^d))$ is the set of all linear subspaces of $ \mathbb{R}^d$.
\end{definition}

As can be seen, in the context of the \textsc{PCA with Outliers} problem, the order of the data is not important relative to the desired subspace. However, our challenge was that we could not directly represent the orderless partitioning of $\mathbb{R}^n$ for subspaces in a suitable arrangement of hyperplanes to apply the proposition~\ref{ArrangementsOfHyperspaces}, and therefore we presented a formulation of the higher degree diagrams on the subspace.

An arrangement of hyperplanes in a $d$-dimensional space $\mathbb{R}^d$ is a subdivision of the space into distinct regions or cells created by a finite set of hyperplanes. Formally, given a set $H$ of $n$ hyperplanes, the arrangement of $H$ divides $\mathbb{R}^d$ into convex polyhedral cells of various dimensions, defined by the intersections and complements of these hyperplanes. Each cell represents a maximal connected subset of $\mathbb{R}^d$ that does not intersect any hyperplane from the set. Arrangements play a crucial role in computational geometry and combinatorics as they provide a framework for analyzing spatial subdivisions. Specifically, arrangements can be used to define Voronoi diagrams. In a Voronoi diagram, the space is divided into regions based on the distance to a specific set of points. Each region, or Voronoi cell, consists of all points closer to one generating point than to any other. By interpreting the bisectors between generating points as hyperplanes, the Voronoi diagram can be viewed as a special case of an arrangement where the hyperplanes are equidistant boundaries between pairs of points. This interpretation allows the powerful combinatorial and geometric tools developed for arrangements to be applied to the study and computation of Voronoi diagrams.

\begin{proposition}
\label{ArrangementsOfHyperspaces}
\textnormal{(\cite{edelsbrunner_constructing_1986}, Theorem 3.3)}
Let $H$ be a set of $n$ hyperplanes in $\mathbb{R}^d$, for $d\geq 2$. Then there is an algorithm that constructs the arrangement of $H$ in $O(n^d)$ time, and this is optimal.
\end{proposition}

\begin{theorem}
\label{thm:VODalgorithm}
Let $\mathcal{S}\subset \mathbb{R}^d$ be a set of $n$ vectors in a $d$-dimensional space. Construction of the {${k}$-degree Voronoi diagram of hyperspaces with respect to vectors $\mathcal{S}$} as defined in definition \ref{VODOfHyperspacesOnPoints} is computable in time $n^{O(d)}$.
\end{theorem}
\begin{proof}
We are going to formulate the boundary of each cell of ${(n-1)}$-degree diagram with a set of linear equations which is going to result in $n^2$ hyperplanes. using proposition \ref{ArrangementsOfHyperspaces}, we are able to compute the expected diagram in $n^{O(d)}$. ${k}$-degree Voronoi diagram is simply constructible by omitting the last $n-k$ indices of each cell.
\\
For every two elements $s_i$ and $s_j$ from $\mathcal{S}$, Let $P_{i,j}$ be a function on the set of all subspaces of the vector space $\mathcal{L}(\mathbb{R}^d)$ to real numbers $\mathbb{R}$ which measures the difference in the distance of $s_i$ and $s_j$ to a given subspace $W$. So, if the projection operator on $W$ is $P_W$ then we can define this as follows:
$$
P_{i,j}(W)=\|P_W\cdot s_i\|-\|P_W\cdot s_j\|
$$ 
By sorting the values of the $P_{i,j}$'s for a specific $W$, we can determine the corresponding cell for this subspace. To elaborate further, since our focus lies only on the sequence of vector distances from $\mathcal{S}$ to $W$, the signs of the equations $P_{i,j}$ suffice. Consequently, a Voronoi cell can be represented by a sign condition based on these equations. Thus, by the continuity of the equations $P_{i,j}$ within the specified domain, we can deduce that a subspace $W$ lies on the boundary of a Voronoi cell only when there exist $i$ and $j$ such that $P_{i,j}(W)=0$. 
As we have $n^2$ equations with a unique rank $d-1$ solution, we have $n^2$ hyperspaces and their arrangments construct the expected diagram. This results in a $(n^2)^d=n^{O(d)}$ time algorithm.
\end{proof}

Figure \ref{fig:vor_optimal_algorithm}
shows a Voronoi diagram of the furthest subspace of the 2nd-degree for 8 points within a two-dimensional space, projected onto one-dimensional subspaces.

\begin{figure}
    \centering
    \includegraphics[height = 9 cm, width = 9 cm]{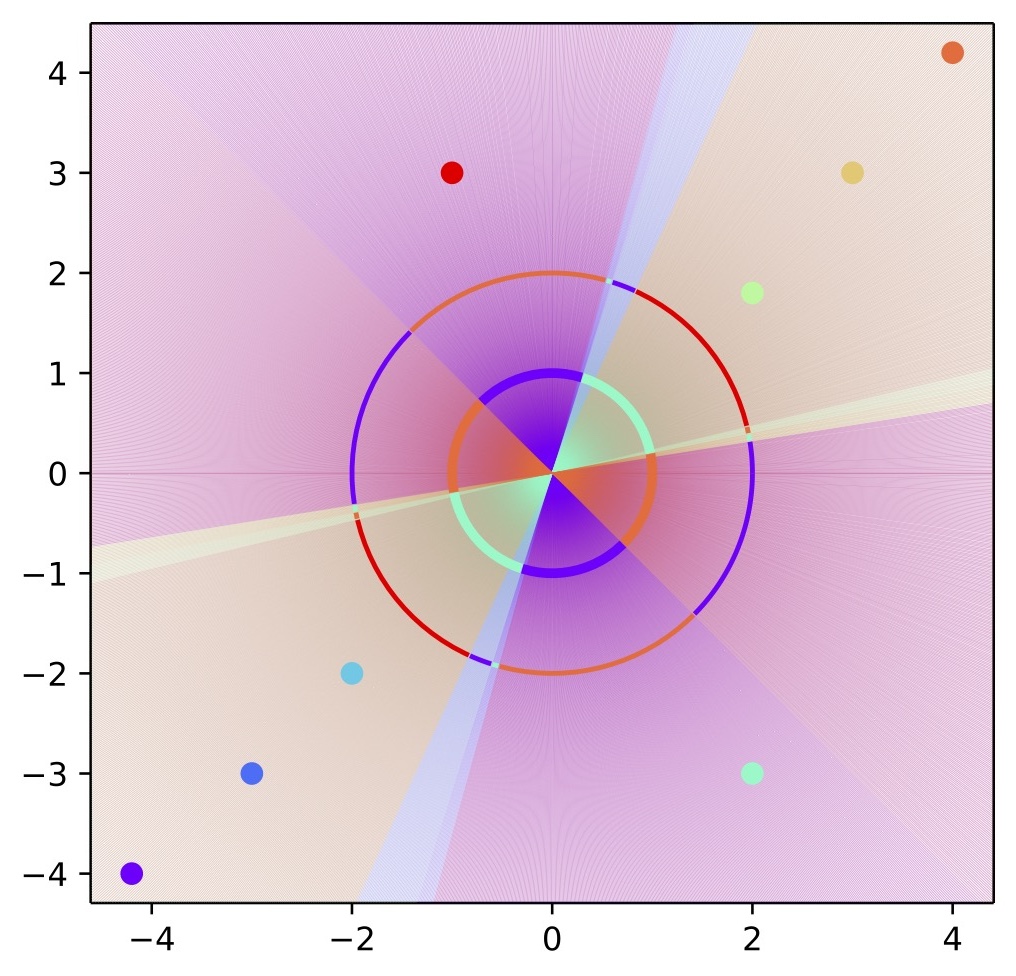}
    \caption{
2nd-degree Voronoi diagram of the furthest subspace for 8 points in $\mathbb{R}^2$ space which can be used to solve PCA with outliers with goal dimension 1 and 2 outliers. The projection of subspaces for each Voronoi cell of the furthest subspace is shown on the unit circle (smaller circle) and the image of subspaces for the Voronoi cell of the second furthest subspace is shown on the circle with radius 2 (larger circle). Space coloring indicates different 2nd-degree cells and subspaces of the same color indicate the same cell. For example, in the above figure, 7 cells are specified, considering that 6 of these eight points belong to exactly one subspace, the number of possible cells is far less than all the possible combinations as outlier data.
}
    \label{fig:vor_optimal_algorithm}
\end{figure}

Also, note that we use the variable $k$ as the number of outliers and the order and degree of Voronoi diagrams interchangeably since the latter is going to denote the number of given outliers.
%
\section{Optimal PCA with Outliers}
In this section, we present our exact algorithm for solving the \textsc{PCA with Outliers} problem in $n^{d+\mathcal{O}(1)}\text{poly}(n,d)$ time. The solution involves obtaining a better subspace embedding for a matrix by removing its outliers. The challenge is that we do not know the actual inlier matrix to compute the scores, as an arbitrary set of rows of the given matrix might be outliers. 
Nevertheless, we can narrow down the available choices such that no possibility for a given combination of outliers, exists.
\\
In earlier work \cite{refined}, \textit{all possible combinations} were generated using methods from real algebraic geometry to model the valid outliers set, and then each cell was tried as a possible candidate. 
The basic idea behind our method is similar to \cite{refined}, but we construct a higher-degree Voronoi diagram out of the given dataset in order to model the outliers set, instead of incorporating algebraic methods.
\begin{theorem}
\label{thm:exactalg}
The \textsc{PCA with Outliers} problem is reducible to solving $n^{d+\mathcal{O}(1)}$ instances of \textsc{PCA}.
\end{theorem}
\begin{proof}
Note that by proposition \ref{lowerbound}, if such an algorithm exists, it should be optimal. In each step, we pick a possible $\mathtt{k}$-tuple of data and label them as outliers, then compute PCA for the rest and measure the relevant error, and finally, depending on the least measured error, we get the output. Nevertheless, the actual problem is the number of \textit{possible $\mathtt{k}$-tuples of outliers}. 
In order to specify each of these possible $\mathtt{k}$-tuples, we need to find their corresponding Voronoi cell. However, there is no need to obtain the exact Voronoi diagram as we only need to know if a set of data can be grouped as outliers, which can be determined by the emptiness of the corresponding cell. 
We choose non-empty cells because when a cell with index tuple $S$ is empty, there is no corresponding subspace with rank $r$ and $S$ as it $\mathtt{(n-k)}$-ordered nearest neighbors to form the required set.
Algorithm~\ref{alg:exactalg} illustrates the above steps.

\begin{algorithm}[h!]
\caption{Computing \textsc{PCA with Outliers} using higher-degree Voronoi Diagram}
\label{alg:exactalg}
\begin{algorithmic}[1]
\Require $X$ ($n$ vectors in $\mathbb{R}^d$), $k$ (\# of outliers), $r$ (Goal dimension) 

\State Compute the $\mathtt{(n-k)}$-degree Voronoi diagram of the rows of $X_{n \times d}$ using theorem~\ref{thm:VODalgorithm}.
\State $\textbf{Cells} \gets$ list of non-empty cells in the computed diagram
\State $\textbf{Index} \gets$ corresponding $k$-tuples of cells in $\textbf{Cells}$
\State $l\leftarrow \infty$ \Comment{keeps the value of the best achieved loss}
\State $S \leftarrow \emptyset$ \Comment{keeps the set of expected outliers}
\For{$\hat{S}\in \textbf{Index}$}
    \If{$\text{loss}(\text{PCA}(X\setminus \hat{S}, r), X\setminus \hat{S})\leq l$}
        \State $l\leftarrow \text{loss}(\text{PCA}(X\setminus \hat{S}, r), X\setminus \hat{S})$
        \State $S\leftarrow \hat{S}$
    \EndIf
\EndFor
\State \Return $\text{PCA}(X\setminus S, r)$
\end{algorithmic}
\end{algorithm}


To show the correctness of our algorithm, we need to show that all possible sets of outliers are explored. Notice that a given subset $S$ of size $k$ is valid as outliers if and only if there is a subspace $L$ of $rank(k)$ with $\mathtt{dist}(x, L)<\mathtt{dist}(s, L)$ for all $x\in X\setminus S$ and $s\in S$. Thus a given subset $S$ of size $k$ is valid as outliers if and only if there is a Voronoi cell corresponding to $X\setminus S$, which proves the correctness of our algorithm.
\\
\end{proof}
Since the number of generated Voronoi cells is bounded as stated in theorem \ref{thm:VODalgorithm}, the total running time of the above algorithm is $n^{d+\mathcal{O}(1)}\text{poly}(n,d)$ which closes the gap in \cite{refined}.


\section{Randomized Algorithm}
In this section, we demonstrate that, rather than explicitly constructing and enumerating every Voronoi cell—a process that can be prohibitively complex—one may instead perform random sampling over candidate subspaces and identify the optimal cell directly.  By leveraging the intrinsic concentration phenomena of high-dimensional geometry, we establish sufficient conditions under which this randomized approach locates the correct Voronoi region with similar time complexity, resolving the need for full diagrammatic enumeration.

We begin by restating the setting of the problem. Let $X\subset \mathbb{R}^d$ be a set of $n$ points contained in the unit ball, and let $S^*\subset X$ be a distinguished subset of $k$ points (The set of optimal outliers). For any vector $v\in\mathbb{R}^d$ and an $r$-dimensional subspace $L\subset \mathbb{R}^d$, define
\[
\mathtt{dist}(v,L) = \|v-P_Lv\|,
\]
where $P_L$ denotes the orthogonal projection onto $L$. Suppose that $L$ minimizes 
\[
\sum_{p\in X\setminus S^*} \mathtt{dist}(p,L).
\]
To tackle this problem efficiently, we assume that the data points in \( X \setminus S^* \) and the outliers in \( S^* \) are well-separated. Consequently, we impose a \emph{separation condition}, formalized in Lemma~\ref{alphagap} to capture the property that outlier points are further
from the best-fit subspace than inlier points. Which asserts the existence of a constant \( \alpha > 0 \) such that for every \( p \in X \setminus S^* \) and \( q \in S^* \), the distance to the subspace \( L \) satisfies the inequality:
\(
\mathtt{dist}(q, L) \ge (1 + \alpha) \cdot \mathtt{dist}(p, L).
\)
This condition ensures that the optimal subspace \( L \) is not significantly influenced by the outliers, making the problem well-posed and enabling us to robustly computes the principal subspace. The separation condition is crucial for both the theoretical guarantees of our algorithm and the subsequent analysis.

\begin{lemma}
\label{alphagap}
\textnormal{\textbf{($\alpha$-Gap Constant)}}
For a set $\text{X} \subset \mathbb{R}^d$, there exists an $\alpha > 0$ such that for an optimal solution of \textsc{PCA with Outliers}, involving a reduced subspace $L$ and a set of outliers $S$, we have for all $s \in S$ and $x \in \text{X} \setminus S$,
$$
\mathtt{dist}(s, L) > (1 + \alpha) \mathtt{dist}(x, L).
$$
\end{lemma}

\begin{proof}
Let $L$ be the optimal subspace for this problem. Since this is the optimal subspace for the inliers, any point classified as an outlier must have a greater distance to $L$ compared to the inliers. Otherwise, if an outlier was closer to $L$ than some inliers, it would be considered part of the main structure captured by $L$ rather than an outlier.
Thus define $d_1$ as the distance of the $(n-k)$th farthest vector to $L$ and $d_2$ as the distance of the $(n-k+1)$ farthest vector to $L$ (i.e., the farthest inlier and the nearest outlier). We can define $\alpha$ as the ratio between these two distances:
$$
\alpha = \frac{d_2 - d_1}{d_1}.
$$
\end{proof}

The \(\alpha\)-gap constant is fundamental in robust Principal Component Analysis (PCA) with outliers, as it guarantees a quantifiable separation between inliers and outliers. This separation is crucial for distinguishing genuine data points from anomalies and thus enhances the robustness of PCA algorithms. In particular, the \(\alpha\)-gap constant generalizes the ideal setting of Robust Subspace Recovery, where inliers lie exactly on a low-dimensional subspace \cite{LermanMaunu2018}, by allowing inliers to reside near the subspace under bounded noise while outliers remain sufficiently distant. Similarly, in probabilistic models such as Robust Probabilistic PCA \cite{Archambeau2006}, inliers drawn from a light-tailed distribution cluster closely around the subspace, whereas outlier distances are amplified in high dimensions.

From a computational perspective, the \(\alpha\)-gap constant facilitates the design of fixed-parameter tractable algorithms for PCA with outliers \cite{fixedpca}. This theoretical soundness translates into practical relevance in applications such as medical imaging and fraud detection, where empirical evidence often suggests that outliers are distinctly separated from inliers (typically with \(\alpha \geq 0.5\)) \cite{Bouwmans2018}.

Also, in high-dimensional spaces, the concentration of measure phenomenon forces inlier distances to the optimal subspace to tightly cluster around a mean value, while outlier distances tend to scale with the ambient dimension. Thus, although the curse of dimensionality is often viewed as a challenge, it in fact reinforces the 
$\alpha$-gap assumption by amplifying the separation between inliers and outliers~\cite{vershynin_high-dimensional_2018, fixedpca}. We are going to leverage this fact to establish lemma \ref{lem:min-dist}.

The following lemma establishes a lower bound on the measure of the set of subspaces \(\hat{L}\) that are close to a fixed subspace \(L\). This measure, quantifies the volume of subspaces \(\hat{L}\) satisfying \(\|P_{\hat{L}} - P_L\| \leq \epsilon\). This result is crucial because it allows us to reason about the probability that a randomly chosen subspace \(\hat{L}\) will lie within a small neighborhood of the optimal subspace \(L\). 

\begin{lemma}\label{lem:volume}
For a fixed subspace $L$ of rank $r$ and $\epsilon>0$ sufficiently small, define
\[
B_\epsilon(L) = \{\text{subspace $\hat{L}$ of rank $r$}: \|P_{\hat{L}}-P_L\|\le \epsilon\}.
\]
Then there exists a constant
\[
c_0(d,r) = J_0\,\frac{\pi^{r(d-r)/2}}{\Gamma\Bigl(\frac{r(d-r)}{2}+1\Bigr)}\, c_2^{-r(d-r)},
\]
depending only on $d$ and $r$, such that
\[
\mu\bigl(B_\epsilon(L)\bigr) \ge c_0(d,r)\,\epsilon^{r(d-r)}.
\]
\end{lemma}

\begin{proof}
A standard local coordinate chart for Grassmanian manifold near $L$ is given by representing any nearby subspace $\hat{L}$ as the graph of a linear map
\(
A : L \to L^\perp,
\)
where $L^\perp$ is the $(d-r)$-dimensional orthogonal complement of $L$. This identifies a neighborhood of $L$ with an open subset of $\mathbb{R}^{N}$, where 
\(
N = r(d-r).
\)
{By perturbation results, the Davis–Kahan “$\sin\Theta$” theorem \cite{DavisKahan1970} and its extension by Stewart–Sun \cite{StewartSun1990}, one obtains that for any linear map $A:L\to L^\perp$ with operator norm $\|A\|$ sufficiently small, there are constants $c_1,c_2>0$, depending only on $d$ and $r$, such that
$$
c_1\,\|A\|\;\le\;\bigl\|P_{\hat L}-P_L\bigr\|\;\le\;c_2\,\|A\|,
$$
}
Thus, $\|P_{\hat{L}}-P_L\|\le \epsilon$ implies $\|A\|\le \epsilon/c_2$. The Euclidean volume in $\mathbb{R}^{N}$ of the ball of radius $\epsilon/c_2$ is
\[
\operatorname{Vol}\Bigl(B_{\epsilon/c_2}(0)\Bigr) = \frac{\pi^{N/2}}{\Gamma\Bigl(\frac{N}{2}+1\Bigr)}\Bigl(\frac{\epsilon}{c_2}\Bigr)^{N}.
\]
{If the Jacobian of the coordinate chart is bounded below by $J_0>0$, then the measure of $B_\epsilon(L)$ satisfies
\[
\mu\bigl(B_\epsilon(L)\bigr) \ge J_0\, \frac{\pi^{r(d-r)/2}}{\Gamma\Bigl(\frac{r(d-r)}{2}+1\Bigr)}\, c_2^{-r(d-r)} \,\epsilon^{r(d-r)}.
\]}
Thus, by setting 
\[
c_0(d,r)=J_0\,\frac{\pi^{r(d-r)/2}}{\Gamma\Bigl(\frac{r(d-r)}{2}+1\Bigr)}\, c_2^{-r(d-r)},
\]
the claim follows. {It remains to show that the Jacobian determinant is bounded below $J_0(d,r)>0$.
At $A=0$, the chart is an isometry up to second‐order terms, so its differential has singular values all equal to 1. By continuity of the singular values on compact neighborhoods, there exists a small radius $\rho>0$ such that for all $|A|\le\rho$, the smallest singular value remains at least, $0<\epsilon_{\rho}\leq 1$.
Hence, by considering the appropriate value $\rho$ such that $\epsilon_\rho=\tfrac12$, on $B_\rho(0)\subset\mathbb{R}^{r(d-r)}$ the Jacobian determinant is $(\tfrac12\bigr)^{r(d-r)}>0$, establishing the desired uniform lower bound.}
\end{proof}

To ensure that such a subspace \(\hat{L}\) preserves the separation between inliers and outliers, we need to understand how perturbations of \(L\) affect the distances of points to the subspace. This is precisely the focus of Lemma~\ref{lem:order}, which bridges the gap between subspace perturbations and the preservation of the inlier-outlier separation. It shows that if \(\hat{L}\) is sufficiently close to \(L\) then the ordering of distances from points to \(\hat{L}\) will preserve the separation between inliers \(X \setminus S^*\) and outliers \(S^*\). This result relies on the \(\alpha\)-gap condition (Lemma~\ref{alphagap}), which ensures that outliers are sufficiently far from \(L\) compared to inliers.

\begin{lemma}\label{lem:order}
Let ${\hat{L}}$ be a subspace of rank $r$ in $\mathbb{R}^d$ such that
\[
\|P_{\hat{L}} - P_L\| \le \epsilon \quad \text{with} \quad \epsilon \le \frac{\alpha}{2} \, \text{dist}_{\min},
\]
where the $\text{dist}_{\min}$ is the minimum distance to the subspace \( L \) over the set of points \( X \setminus S^* \), formally defined as $
\text{dist}_{\min} = \min_{p \in X \setminus S^*} \mathtt{dist}(p, L).
$
Then for every $p\in X\setminus S^*$ and every $q\in S^*$,
\(
\mathtt{dist}(p,{\hat{L}}) < \mathtt{dist}(q,{\hat{L}}).
\)
Meaning that the $(n-k)$-nearest neighbor set of ${\hat{L}}$ is exactly $X\setminus S^*$.
\end{lemma}

\begin{proof}
For any $v\in X$, one has
\[
|\mathtt{dist}(v,{\hat{L}})-\mathtt{dist}(v,L)| \le \|P_{\hat{L}}-P_L\|.
\]
Thus, for $p\in X\setminus S^*$ and $q\in S^*$,
\[
\mathtt{dist}(p,{\hat{L}}) \le \mathtt{dist}(p,L)+\epsilon,\quad \mathtt{dist}(q,X) \ge \mathtt{dist}(q,L)-\epsilon.
\]
By the $\alpha$-gap constant (Lemma \ref{alphagap}),
\[
\mathtt{dist}(q,L)-\mathtt{dist}(p,L) \ge \alpha\, \mathtt{dist}(p,L) \ge \alpha\, \text{dist}_{\min}.
\]
Therefore, if 
\(
2\epsilon \le \alpha\, \text{dist}_{\min},
\)
then
\[
\mathtt{dist}(q,L)-\epsilon \ge \mathtt{dist}(p,L)+\epsilon,
\]
implying $\mathtt{dist}(p,{\hat{L}}) < \mathtt{dist}(q,{\hat{L}})$.
\end{proof}

By combining Lemma~\ref{lem:volume} and Lemma~\ref{lem:order}, we can conclude that with a fine probability, a random subspace \(\hat{L}\) will preserve the inlier-outlier separation, provided \(\epsilon\) is chosen appropriately.

However, Lemma~\ref{lem:order} depends critically on the minimum distance \(\text{dist}_{\min}\) from the inliers \(X \setminus S^*\) to the subspace \(L\). Specifically, the perturbation threshold \(\epsilon = \frac{\alpha}{2} \text{dist}_{\min}\) must be small enough to ensure that the distances of inliers and outliers to \(\hat{L}\) are not perturbed too much. This raises the question: what we expect the value of \(\text{dist}_{\min}\) to be?

The following lemma answers this question by providing a probabilistic lower bound on \(\text{dist}_{\min}\). This result is derived from a geometric argument about the distribution of points in high-dimensional space. The key insight is that, in high dimensions, points are unlikely to cluster too closely around any fixed subspace \(L\), and thus \(\text{dist}_{\min}\) is bounded below and ensures that the perturbation threshold is sufficiently small to apply Lemma~\ref{lem:order}, thereby guaranteeing the preservation of the inlier-outlier separation. In particular, we show that with high probability $\mathrm{dist}_{\min}\;\ge\;c_1\,n^{-1/(d-r)}$.

\begin{lemma}\label{lem:min-dist}
Let $S = \{p_1,\dots,p_n\}\subset\mathbb R^d$ be drawn i.i.d. from the uniform distribution on the unit ball in $\mathbb R^d$.  Let $L\subset\mathbb R^d$ be any fixed $r$-dimensional subspace, and define
$$
\mathrm{dist}_{\min}
=\min_{p\in S}\mathtt{dist}(p,L)
=\min_{p\in S}\bigl\|P_{L^\perp}p\bigr\|.
$$
Then there is a constant $c_1=c_1(d,r)>0$ such that
$$
\Pr\bigl[\mathrm{dist}_{\min}\ge c_1\,n^{-1/(d-r)}\bigr]\;=\;1 - O\bigl(n^{-1}\bigr).
$$
\end{lemma}

\begin{proof}
Since the law of a uniform point $p$ in the $d$-ball is invariant under orthogonal transformations, the marginal distribution of its $L^\perp$-component depends only on the radius.  Concretely, for any $\delta\in[0,1]$,

$$
   \Pr\bigl(\|P_{L^\perp}p\|\le\delta\bigr)
   \;=\;
   \frac{Vol\bigl(B^{d-r}_\delta\bigr)}{Vol\bigl(B^d_1\bigr)}
   \;=\;\delta^{\,d-r},
$$

since the volume of a $k$-dimensional Euclidean ball of radius $\delta$ scales like $\delta^k$ .

Now we can conclude the result using a Tail bound and union bound. Let $\displaystyle X_p=\|P_{L^\perp}p\|$. Then $\Pr\bigl(X_p\le\delta\bigr)=\delta^{\,d-r}$. By a union bound over the $n$ independent samples, $\Pr\bigl(\exists\,p\in S:\,X_p\le\delta\bigr)\;\le\;n\,\delta^{\,d-r}$.

Thus $\Pr\bigl(\mathrm{dist}_{\min}\le\delta\bigr)\le n\,\delta^{\,d-r}$. To make this failure probability $ \le n^{-1}$, choose $\delta = c_1\,n^{-1/(d-r)}$ with $c_1$ a suitable constant, yielding $\Pr(\mathrm{dist}_{\min}\ge\delta)\ge1-O(n^{-1})$. This completes the proof that, under the stated sampling model, the minimum distance of $S$ to any fixed $r$-subspace $L$ exceeds $c_1\,n^{-1/(d-r)}$ with high probability.
\end{proof}

It is important to note that the bound holds provided that the number of points \( n \) is not large enough to form an adversarial distribution. An adversarial distribution would require at least \( n_{\text{max}} = O\left( \delta^{-(d - r)} \right) \) points to densely pack the \( (d - r) \)-dimensional space and break the lower bound on \( \text{dist}_{\min}(S, L) \). If \( n \) is smaller than this threshold, the points are unlikely to form an adversarial configuration, and the bound on \( \text{dist}_{\min}(S, L) \) holds with high probability.

Thus, for \( n \) sufficiently small, the minimum distance between the points in \( S \) and the subspace \( L \) is at least \( c_1 \, n^{-1/(d - r)} \) with high probability.

With the results of Lemmas~\ref{lem:volume},~\ref{lem:order}, and~\ref{lem:min-dist}, in place, we are now ready to derive the probability that a random subspace \(\hat{L}\) correctly identifies the inliers \(X \setminus S^*\) as the \((n-k)\)-nearest neighbors. This result is significant because it provides a quantitative guarantee for the success of the randomized algorithm~\ref{alg:randomized} in identifying the optimal subspace and removing outliers. The constant \(C(d,r,\alpha)\) aggregates the contributions of previously stated lemmas, reflecting the interplay between the geometry of high-dimensional spaces, the separation condition, and the probabilistic nature of the algorithm.

\begin{theorem}\label{thm:knn}
Under assumptions of lemmas ~\ref{lem:volume},~\ref{lem:order}, and~\ref{lem:min-dist}, if ${\hat{L}}$ is chosen uniformly at random from the set of all $r$-dimensional subspace of $\mathbb{R}^d$, then
\[
\Pr\Bigl[X\setminus S^*\text{ is the }(n-k)\text{-NN set of }{\hat{L}}\Bigr] \ge C(d,r,\alpha)\, n^{-r}
\]
where $C$ is a constant depending on $d$, $r$ and $\alpha$. Explicitly, 
\(
C(d,r,\alpha) = c_0(d,r)\,\left(\frac{\alpha\, c_1}{2}\right)^{r(d-r)}.
\)
\end{theorem}

\begin{proof}
By Lemma~\ref{lem:min-dist}, with high probability,
\(
\text{dist}_{\min} \ge c_1\, n^{-1/(d-r)}.
\)
Set
\[
\epsilon =  \frac{\alpha}{2}\, c_1\, n^{-1/(d-r)} \leq \frac{\alpha}{2}\, \text{dist}_{\min}.
\]
Then, by Lemma~\ref{lem:order}, if $\|P_{\hat{L}}-P_L\|\le \epsilon$, the $(n-k)$-nearest neighbor set of ${\hat{L}}$ is exactly $X\setminus S^*$. By Lemma~\ref{lem:volume}, the probability that
\(
\|P_{\hat{L}}-P_L\|\le \epsilon
\)
is at least
\(
c_0(d,r)\,\epsilon^{r(d-r)}.
\)
Substitute the bound for $\epsilon$:
\[
\begin{split}
\Pr\Bigl[X\setminus S^*=(n-k)\text{-NN of }{\hat{L}}\Bigr]
&\ge c_0(d,r) \left(\frac{\alpha}{2}\, c_1\, n^{-1/(d-r)}\right)^{r(d-r)}
\\ &= c_0(d,r)\,\left(\frac{\alpha\, c_1}{2}\right)^{r(d-r)} n^{-r}.
\end{split}
\]
Defining
\[
C(d,r,\alpha) = c_0(d,r)\,\left(\frac{\alpha\, c_1}{2}\right)^{r(d-r)},
\]
the theorem is proved.
\end{proof}

This theorem provides the theoretical foundation for the analysis of our randomized algorithm~\ref{alg:randomized}. To achieve a high-confidence guarantee, the algorithm performs multiple independent trials, each involving the random selection of a subspace \(\hat{L}\) and the computation of distances to this subspace. By repeating this process many times, the algorithm ensures that the probability of at least one successful trial is at least \(1-\delta\).

The time complexity of the algorithm is dominated by the number of trials \(t\) and the cost of each trial, which includes sampling a random subspace, computing distances, and performing PCA. 

\begin{algorithm}[H]
\caption{Randomized PCA with Outliers}
\label{alg:randomized}
\begin{algorithmic}[1]
\Require $X$ ($n$ vectors in $\mathbb{R}^d$), $t$ (\# of trials), $k$ (\# of outliers), $r$ (Goal dimension) 
\State $\tilde{l}\leftarrow \infty$ \Comment{keeps the value of the best achieved loss}
\State $\tilde{S} \leftarrow \emptyset$ \Comment{keeps the set of expected outliers}
\While{$t\geq 0$}
    \State $t\leftarrow t-1$.
    \State $\hat{L}\leftarrow$ uniformly sampled subspace of rank $r$.
    \State $\hat{S}\leftarrow$ the set of $k$ vectors from X with maximum distance to $\hat{L}$.
    \If{$\text{loss}(\text{PCA}(X\setminus \hat{S}, r), X\setminus \hat{S})\leq\tilde{l}$}
        \State $\tilde{l}\leftarrow \text{loss}(\text{PCA}(X\setminus \hat{S}, r), X\setminus \hat{S})$
        \State $\tilde{S}\leftarrow \hat{S}$
    \EndIf
\EndWhile
\State \Return $\text{PCA}(X\setminus \tilde{S}, r)$
\end{algorithmic}
\end{algorithm}

\begin{theorem}
\label{thm:randomized-alg}
Let \( X \subset \mathbb{R}^d \) be a set of \( n \) vectors, $r\leq d$, and $\delta\in (0,1)$ be given parameters. Then, the randomized algorithm~\ref{alg:randomized} returns the optimal solution to the \textsc{PCA with Outliers} problem with probability at least $1-\delta$ in time 
\[
O\Bigl( \frac{n^r \log(1/\delta)}{C(d,r,\alpha)}\text{poly}(d,r) \Bigr)
\]
where $\alpha$ is the same as in Lemma \ref{alphagap} and $C(d,r,\alpha)$
is an explicit constant (depending only on \(d\), \(r\), and \(\alpha\)) as in Theorem ~\ref{thm:knn}.
\end{theorem}

\begin{proof}

{
By Theorem~\ref{thm:knn}, a single random draw of an $r$-dimensional subspace lands in the optimal Voronoi cell with probability $p\geq\frac{C}{n^r}$. We perform $T$ independent trials of the sampling. Let the binary variable $X_i=1$ iff the $i$th trial succeeds.
Then $S = \sum_{i=1}^T X_i$ has $\mathbb{E}[S] = T p$.  We want $\Pr[S \leq 0]\le \delta$. By the Chernoff bound, for any $\epsilon\in(0,1)$,
$$
\Pr\bigl[S \le (1-\epsilon)\mathbb{E}[S]\bigr]
\le
\exp\Bigl(-\tfrac{\epsilon^2}{2}\mathbb{E}[S]\Bigr)\,.
$$
In particular, set $\epsilon=1$.  Then $(1-\epsilon)\,\mathbb{E}[S]=0$, and
$\Pr\bigl[S \le 0\bigr]\le\exp\bigl(-\tfrac12\,T p\bigr).$
Hence, choosing
$$
T \;=\;\Bigl\lceil \frac{2\,n^r\ln(1/\delta)}{C}\Bigr\rceil
\;=\;O\!\bigl(n^r\log(1/\delta)/C\bigr)
$$
guarantees success with probability at least $1-\delta$.

Each trial requires $O\bigl(n\,\mathrm{poly}(d,r)\bigr)$ time, so the total running time is
$$
T \times O\bigl(n\,\mathrm{poly}(d,r)\bigr)
\;=\;
O\!\Bigl(\frac{n^r\log(1/\delta)}{C}\,\mathrm{poly}(d,r)\Bigr),
$$
}
\end{proof}

Our analysis employs a conservative lower bound on the volume of the optimal Voronoi cell, which yields a cautious estimate for the algorithm’s success probability. In practice, the actual probability of sampling the optimal cell and the true success rate substantially exceeds this baseline bound.


\section{Conclusion}
In this paper, we introduced a new algorithm for the problem of PCA in the presence of outliers, leveraging the concept of higher-degree Voronoi diagrams to effectively detect outliers. Our approach builds on previous work by offering an exact solution with a complexity of $n^{d+\mathcal{O}(1)}\text{poly}(n,d)$, which is more efficient than the solutions in \cite{fixedpca} and \cite{refined}. Furthermore, our method provides a more intuitive framework to solve the \textsc{PCA with Outliers} problem by clearly delineating the roles of data partitioning and outlier detection.
We also presented a randomized version of our algorithm, showing that achieving the optimal subspace can be possible through uniform sampling. While this randomized solution is currently less efficient than the one in \cite{fixedpca}, we suggest that adopting Gaussian sampling could enhance its accuracy and efficiency, as our main future study direction. Our paper focused on theoretical aspect of the robust PCA, however we intend to implement our constructive algorithms for use in applications.
Another direction is exploring the dual concept of Voronoi diagrams, specifically Delaunay triangulation, which could provide further insights into the \textsc{PCA with Outliers} problem. The constructive algorithm for Delaunay triangulation presented by~\cite{delaun} offers a promising avenue for investigation.
Finally, we would like to link our work to other variants of PCA with outliers, namely online PCA~\cite{online, NIPS2013_8f121ce0}. In particular, we are interested in Bandit PCA~\cite{pmlr-v99-kotlowski19a}, where there are feedback loops for having better sequential projections, and where data samples are noisy (have outliers).

\bibliographystyle{plainurl}
\bibliography{bib.bib} 




\end{document}